\documentclass{amsart}
\usepackage[margin=1.5cm]{geometry}
\usepackage{graphicx} 
\usepackage[british]{babel}
\usepackage[ruled,linesnumbered]{algorithm2e}
\usepackage{enumerate,bm,amsfonts,amsmath,amssymb,amsthm,amsfonts,fourier,booktabs,blkarray,xparse,soul,comment,url,nicefrac,xcolor,bm,tikz,pgfplots,subfigure,hyperref,url}
\definecolor{myblueR2}{RGB}{0, 0, 0}
\definecolor{myblue}{RGB}{0, 0, 0}

\newtheorem{theorem}{Theorem}

\newtheorem{proposition}{Proposition}

\usetikzlibrary{arrows,plotmarks,decorations.markings,trees,shapes,pgfplots.groupplots,automata,circuits}
\usepgfplotslibrary{statistics}
\tikzset{dot/.style = {circle, fill, minimum size=#1,inner sep=0pt, outer sep=0pt, fill, circle},dot/.default = 6pt}
\tikzset{dot2/.style = {circle, fill, color=black!40,minimum size=6pt,inner sep=0pt, outer sep=0pt, fill, circle}}
\tikzstyle{a}=[->,>=stealth,dashed]
\tikzstyle{a2}=[->,>=stealth]
\tikzstyle{a3}=[<->,>=stealth]
\tikzstyle{nodo}=[ellipse,draw=black!100,fill=black!0,line width=.7pt,minimum width=1.0cm,minimum height=0.8cm,text width=1cm,text centered]
\tikzstyle{nodo2}=[ellipse,draw=black!100,fill=black!10,line width=.7pt,minimum width=1.0cm,minimum height=0.8cm,text width=1cm,text centered]
\tikzstyle{nodo3}=[ellipse,draw=black!100,fill=black!30,line width=.7pt,minimum width=1.0cm,minimum height=0.8cm,text width=1cm,text centered]
\tikzstyle{Qnodo}=[ellipse,draw=black!100,fill=black!10,line width=.7pt,minimum width=1.2cm,minimum height=.7cm]
\tikzstyle{arco}=[draw=black!80,line width=.7pt, postaction={decorate}, decoration={markings,mark=at position 1.0 with {\arrow[ draw=black!80,line width=.7pt]{>}}}]
\tikzstyle{decision} = [rectangle, draw, fill=black!100,text=white, text width=4.5em, text badly centered, node distance=3cm, minimum height=3em]
\tikzstyle{block} = [rectangle, draw, fill=blue!20, text width=5em, text centered, rounded corners, minimum height=3em]
\tikzstyle{line} = [draw, -latex']
\tikzstyle{cloud} = [draw, ellipse,fill=red!20, node distance=3cm, minimum height=2em]

\pgfplotsset{legend image with text/.style={
legend image code/.code={%
\node[anchor=center] at (0.3cm,0cm) {#1};}},}

\usepackage{mathtools} 
\usepackage{booktabs} 
\usepackage{tikz,bm} 
\title{A Note on Bayesian Networks with Latent Root Variables}
\author{Marco Zaffalon}
\author{Alessandro Antonucci}
\address{IDSIA, Lugano (Switzerland)}
\email{\{zaffalon,alessandro\}@idsia.ch}
\date{February 2024}
\begin{document}
\maketitle
\begin{abstract}
We characterise the likelihood function computed from a Bayesian network with latent variables as root nodes. We show that the marginal distribution over the remaining, manifest, variables also factorises as a Bayesian network, which we call \emph{empirical}. A dataset of observations of the manifest variables allows us to quantify the parameters of the empirical Bayesian net. We prove that (i) the likelihood of such a dataset from the original Bayesian network is dominated by the global maximum of the likelihood from the empirical one; and that (ii) such a maximum is attained if and only if the parameters of the Bayesian network are  consistent with those of the empirical model.
\end{abstract}
\section{Introduction}\label{sec:intro}
Bayesian networks are a popular class of probabilistic graphical models allowing for a compact specification of joint distributions \cite{koller2009}. Given a complete dataset of observations of its variables, a Bayesian network of a given structure allows for closed-form learning of its parameters. This reflects the unimodality of the likelihood of the data from the model. The situation differs in the presence of latent variables, whence of incomplete observations: these give rise to local maxima in the likelihood, thus preventing the parameters from being learnt optimally \cite{wang2006}. In this  note, we show that if the latent variables correspond to the \emph{root} (i.e., parentless) nodes of the network, the global maximum of the likelihood function can still be identified; this allows us to verify when the learning of the parameters is optimal. The result extends an analogous one proved in the special case where the conditional distributions of non-root nodes (we call them \emph{internal}) are deterministic \cite[Thm.~1]{zaffalon2021,zaffalon2023a}.

The note is organised as follows. In Sect.~\ref{sec:bn}, we consider Bayesian networks with latent root nodes and show how their marginal distribution of the remaining manifest variables also factorises as a Bayesian network. We use this in Sect.~\ref{sec:lik} to characterise the global maximum of the likelihood. The proof of our main result is based on a transformation making the internal nodes deterministic functions of their parents by means of auxiliary root variables. Conclusions are in Sect.~\ref{sec:conc}.

\section{Marginalising Root Variables in Bayesian Networks}\label{sec:bn}
We focus on categorical variables. If $X$ is a variable, its generic state is $x$, the set of its possible states $\Omega_X$, and a \emph{probability mass function} (PMF) over $X$ is $P(X)$. Consider a joint variable $\bm{X}:=(X_1,\ldots,X_n)$. A \emph{Bayesian network} (BN) $B$ allows for a compact specification of a joint PMF $P(\bm{X})$ \cite{koller2009}. 

Consider a directed acyclic graph $\mathcal{G}$, whose nodes are in a one-to-one correspondence with the variables in $\bm{X}$. Assuming the graph to represent conditional independence relations according to the \emph{Markov condition}, i.e., any variable is stochastically independent of its non-descendants non-parents given its parents, we obtain the following factorisation:
\begin{equation}
P(\bm{x}) = \prod_{X \in \bm{X}} P(x|\mathrm{pa}_{X})\,,
\end{equation}
for each joint state $\bm{x}$, where the states $(x,\mathrm{pa}_X)$ are those consistent with $\bm{x}$ and $\mathrm{Pa}_X$ are the \emph{parents}, i.e., predecessors of $X$ according to $\mathcal{G}$. Given $\mathcal{G}$, a BN specification corresponds to the collection $B:=\{P(X|\mathrm{Pa}_X)\}_{X\in\bm{X}}$, where the conditional PMFs 
$P(X|\mathrm{Pa}_X):=\{P(X|\mathrm{pa}_X)\}_{\mathrm{pa}_X\in\Omega_{\mathrm{Pa}_X}}$ define a \emph{conditional probability table} (CPT) for $X$, for each $X\in\bm{X}$. A CPT becomes a single, unconditional, PMF for variables associated with root nodes.

Let $\bm{Y}\subseteq \bm{X}$ denote the variables associated with the root nodes of $\mathcal{G}$, while $\bm{Z}:=\bm{X}\setminus \bm{Y}$ are those corresponding to the internal nodes. The connected components of the graph obtained by removing from $\mathcal{G}$ any arc connecting pairs of internal variables induce a partition of $\bm{Z}$, whose elements are called \emph{c-components}. The following result holds.
\begin{proposition}\label{prop:tian}
The marginal PMF over the internal variables $P(\bm{Z})$ factorises as:
\begin{equation}\label{eq:empirical}
P(\bm{z})=\prod_{Z \in \bm{Z}} P(z|\bm{w}_Z)\,,
\end{equation}
for each $\bm{z}$, where the states $(z,\bm{w}_z)$ are those consistent with $\bm{z}$, and $\bm{W}_Z$ denotes the union of the internal variables of the same c-component of $Z$ and their internal parents, after the removal of $Z$ and its descendants. 
\end{proposition}
The factorisation in Eq.~\eqref{eq:empirical} has been proved in \cite{tian2002studies} in the context of BNs whose internal CPTs are \emph{deterministic} (i.e., their PMFs assign all the mass to a single state). The proof only exploits the factorisation properties induced by the Markov condition, and remains valid for general BNs.

\section{Likelihood Characterisation}\label{sec:lik}
Assume that the variables in $\bm{Y}$ are \emph{latent}, and a dataset $\mathcal{D}$ of observations of $\bm{Z}$ is available. Prop.~\ref{prop:tian} says that the marginal PMF $P(\bm{Z})$ over the internal, and hence \emph{manifest}, variables of $B$ factorises as a second BN. From such a BN, the log-likelihood of $\mathcal{D}$ takes a simple multinomial form:
\begin{equation}\label{eq:multilik}
\lambda(\theta_{\bm{Z}}) := \sum_{\bm{z}} n(\bm{z}) \sum_{Z\in\bm{Z}} \log \theta_{z|\bm{w}_Z}\,,
\end{equation}
where $n(\cdot)$ denotes the frequencies in $\mathcal{D}$ of its argument, while $\theta_{\bm{Z}}$ is the collection of all $\theta_{z|\bm{w}_Z}:=P(z|\bm{w}_Z)$, for each $z$ and $\bm{w}_Z$, with $Z\in\bm{Z}$. The function in Eq.~\eqref{eq:multilik} is unimodal, with a single (global) maximum $\lambda^*$ attained at
\begin{equation}\label{eq:freqs}
\theta_{z|\bm{w}_Z}:=\frac{n(z,\bm{w}_Z)}{n(\bm{w}_Z)}\,,
\end{equation}
for each $z$ and $\bm{w}_V$, and $Z\in\bm{Z}$ \cite{koller2009}. We call \emph{empirical} the BN over $\bm{Z}$ factorising as in Eq.~\eqref{eq:empirical}, and quantified as in Eq.~\eqref{eq:freqs}.

If the conditional probabilities we infer from the original BN $B$ over $\bm{X}$ satisfy Eq.~\eqref{eq:freqs}, we say that the dataset $\mathcal{D}$ is \emph{compatible} with $B$.

Now consider the log-likelihood of $\mathcal{D}$ from $B$, i.e.,
\begin{equation}\label{eq:lik}
l(\theta_{\bm{Y}},\theta_{\bm{Z}}')
:=
\sum_{\bm{z}} n(\bm{z}) \log P(\bm{z})
=
\sum_{\bm{z}} n(\bm{z}) \log \sum_{\bm{y}} \left[ \prod_{Y\in\bm{Y}} \theta_{y} \prod_{Z\in\bm{Z}} \theta_{z|\mathrm{pa}_{Z}}
\right]
\,,
\end{equation}
where $\theta_{\bm{Y}}$ denote the collection of all $\theta_y:=P(y)$ for each $y$, with $Y\in\bm{Y}$, and $\theta_{\bm{Z}}'$ that of all $\theta_{z|\mathrm{pa}_Z}:=P(Z|\mathrm{pa}_Z)$ for each $z$ and $\mathrm{pa}_Z$, with $Z\in\bm{Z}$.
 
Unlike $\lambda(\theta_{\bm{Z}})$ in Eq.~\eqref{eq:multilik}, $l(\theta_{\bm{Y}},\theta_{\bm{Z}}')$ can in principle be subject to local maxima, because of the marginalisation of the latent variables in Eq.~\eqref{eq:lik}. Yet, the following result gives a characterisation of its global maximum through a connection with the notion of compatibility.
\begin{theorem}\label{th:unimodal}
Given a dataset $\mathcal{D}$ of observations of $\bm{Z}$, it holds that $l(\theta_{\bm{Y}},\theta_{\bm{Z}}')\leq\lambda^*$. Moreover, $\mathcal{D}$ is compatible with $B$ if and only if
$l(\theta_{\bm{Y}},\theta_{\bm{Z}}')=\lambda^*$.
\end{theorem}
\begin{proof}
Given the BN $B$ over $\bm{X}$ in Sect.~\ref{sec:bn}, for each internal node $Z\in\bm{Z}$, do the following (check Figs.~\ref{fig:bn},~\ref{fig:scm}):
\begin{enumerate}[(i)]
    \item Add an auxiliary root node $U_Z$ as a parent of $Z$.
    \item Let $|\Omega_{U_Z}|\coloneqq|\Omega_Z|^{|\Omega_{\mathrm{Pa}_Z}|}$ and define the marginal PMF $P(U_Z)$ as the joint PMF induced by the conditional PMFs over $Z$ in the CPT for the different values of $\mathrm{Pa}_Z$ under an independence assumption:
    \begin{equation}\label{eq:map}
P(U_Z):=\prod_{\mathrm{pa}_Z\in\Omega_{\mathrm{Pa}_Z}}P(Z|\mathrm{pa}_Z)\,.
    \end{equation} 
    Note that each state $u_Z$ corresponds to a collection of states of $Z$, one per each state of $\mathrm{Pa}_Z$. Denote such a vector by $u_Z^{-1}$. 
    \item Replace the original CPT of $Z$, i.e., $P(Z|\mathrm{Pa}_Z)$, with a deterministic CPT $P(Z|\mathrm{Pa}_Z,U_Z)$ that assigns all the mass to the state of $Z$ associated with $(u_Z,\mathrm{pa}_Z)$:
    \begin{equation*}
P(Z|\mathrm{pa}_Z,u_Z):= 
\begin{cases}
    1& \text{if } z=u^{-1}_Z[\mathrm{pa}_Z]\\
    0              & \text{otherwise\,.}
\end{cases}
\end{equation*}     
\end{enumerate}
The above transformation returns a new BN $B'$ over $\bm{X} \cup \{U_Z\}_{Z\in\bm{Z}}$ whose internal CPTs are deterministic.\footnote{We thank Rina Dechter and Denis Mau\'a for pointing us to such a transformation that appears to be common knowledge but seemingly not formalised in the literature (or present in old work that we could not find).}
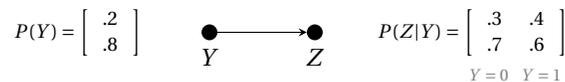
\begin{figure}[htp!]
\centering
\begin{tikzpicture}[scale=0.7]
\node[dot,label=below:{$Y$}] (y)  at (0,0) {};
\node[dot,label=below:{$Z$}] (z)  at (2,0) {};
\draw[a2] (y) -- (z);
\node[] at (-2.5,0) {\footnotesize $P(Y)=\left[\begin{array}{c}.2\\.8\end{array}\right]$};
\node[] at (5,0) {\footnotesize $P(Z|Y)=\left[\begin{array}{cc}.3 & .4\\.7&.6\end{array}\right]$};
\node[] at (5.3,-0.8) {\scriptsize\color{black!50}{$Y=0$}};
\node[] at (6.3,-0.8) {\scriptsize\color{black!50}{$Y=1$}};
\end{tikzpicture}
\caption{A BN over two Boolean variables and its CPTs.}\label{fig:bn}
\end{figure}

\begin{figure}[htp!]
\centering
\begin{tikzpicture}[scale=0.7]
\node[dot,label=below:{$Y$}] (y)  at (0,0) {};
\node[dot,label=below:{$Z$}] (z)  at (2,0) {};
\draw[a2] (y) -- (z);
\node[dot,label=above:{$U_Z$}] (uz)  at (2,1.5) {};
\draw[a2] (uz) -- (z);
\node[] at (-2.5,0) {\footnotesize $P(Y)=\left[\begin{array}{c}.2\\.8\end{array}\right]$};
\node[] at (6,2) {\footnotesize $P(U_Z)=\left[\begin{array}{c}0.3 \times 0.4\\0.3 \times 0.6\\0.7 \times 0.4\\0.7 \times 0.6\\ \end{array}\right]=\left[\begin{array}{c}0.12\\0.18\\0.28\\0.42\\ \end{array}\right]$};
\node[] at (7,0) {\footnotesize $P(Z|Y,U_Z)=\left[\begin{array}{cccc|cccc}1&1&0&0&1&0&1&0\\
0&0&1&1&0&1&0&1\end{array}\right]$};
\node[] at (6.7,-0.8) {\scriptsize\color{black!50}{$Y=0$}};
\node[] at (9.5,-0.8) {\scriptsize\color{black!50}{$Y=1$}};
\node[] at (6.7,-1.3) {\scriptsize\color{black!50}{$U_Z=0$ $1$ $2$ $3$}};
\node[] at (9.5,-1.3) {\scriptsize\color{black!50}{$U_Z=0$ $1$ $2$ $3$}};
\end{tikzpicture}
\caption{A BN with deterministic internal CPTs.}\label{fig:scm}
\end{figure}
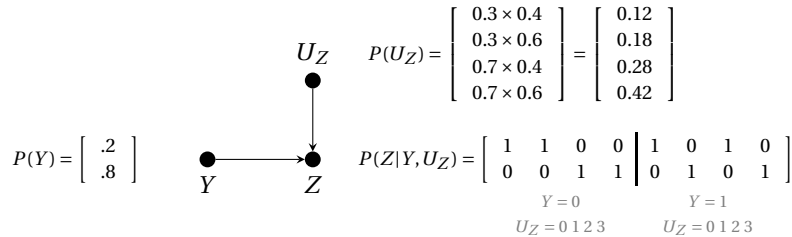

It is straightforward to check that, in the model $B'$ in Fig.~\ref{fig:scm}:
\begin{equation*}
P(Z=0|Y=0) = \sum_{u_Z} P(Z=0|Y=0,u_Z) \cdot P(u_Z) = 0.12 + 0.18 = 0.3\,,
\end{equation*}
that is the same value of $P(Z=0|Y=0)$ in $B$. We can similarly retrieve all the elements of $P(Z|Y)$. As $P(Y)$ is the same in both BNs, we also have the same joint $P(Y,Z)$ and, hence, the same $P(Z)$. In the general case, by considering the parameters of $B'$ on the left-hand side and those of $B$ on the right, we can similarly obtain:
\begin{equation}\label{eq:comp}
\sum_{u_Z} P(z|\mathrm{pa}_Z,u_Z) \cdot P(u_Z) = P(z|\mathrm{pa}_Z)\,,
\end{equation}
for each $z$ and $\mathrm{pa}_Z$, with $Z\in\bm{Z}$. This means that $P(\bm{Y},\bm{Z})$ is the same for both $B$ and $B'$. Thus, the BNs $B$ and $B'$ define the same joint PMF $P(\bm{Z})$ over their internal nodes, but the internal CPTs of $B'$ are deterministic, while those of $B$ are, in general, not. 

Recall that \cite[Thm.~1]{zaffalon2021,zaffalon2023a} proved the result we are after in the special case where internal CPTs are deterministic and fixed, and regarding only the marginal PMFs on the root nodes as parameters. The result therefore holds for $B'$, provided that we regard the log-likelihood of $\mathcal{D}$ from $B'$ as:
\begin{equation}\label{eq:ll}
l'(\theta_{\bm{Y}},\theta_{\bm{U}}):=
\sum_{\bm{z}} n(\bm{z}) \log \sum_{\bm{y}} \left[ \prod_{Y\in\bm{Y}} \theta_{y} \cdot \prod_{Z\in\bm{Z}} \sum_{u_Z} p(z|\mathrm{pa}_{Z},u_Z) \cdot \theta_{u_Z}
\right]
\,,
\end{equation}
where $\theta_{\bm{U}}$ is the collection of all $\theta_{u_Z}:=P(u_Z)$ for each $u_Z$, with $Z\in\bm{Z}$. 

Because of Eq.~\eqref{eq:comp}, we have that Eq.~\eqref{eq:lik} coincides with Eq.~\eqref{eq:ll} provided that we map $\theta_{\bm{Z}}'$ to $\theta_{\bm{U}}$ by Eq.~\eqref{eq:map}. As $B$ and $B'$ share the same $P(\bm{Z})$, we also have that the compatibility of $\mathcal{D}$ with $B$ is equivalent to that of $B'$.  Finally, note that the two BNs also share the same empirical BN, and hence the same global maximum $\lambda^*$. 

To prove the first statement of the theorem, assume, \emph{ad absurdum}, $l(\theta_{\bm{Y}},\theta_{\bm{Z}}')>\lambda^*$. This would imply, for the corresponding value of $\theta_{\bm{U}}$, $l'(\theta_{\bm{Y}},\theta_{\bm{U}})>\lambda^*$. But this is impossible, because $B'$ satisfies the theorem.

The second statement of the theorem says that $\mathcal{D}$ is compatible with $B$ if and only if $l(\theta_{\bm{Y}},\theta_{\bm{Z}}')=\lambda^*$. The condition on the log-likelihood can be written as
$l'(\theta_{\bm{Y}},\theta_{\bm{U}})=\lambda^*$, while the compatibility condition can be equivalently required for $B'$. As $B'$ satisfies the theorem, we have the thesis.
\end{proof}

\section{Conclusions}\label{sec:conc}
We have provided a characterisation of the likelihood function for Bayesian nets with latent variables as root nodes. Analogously to what is done in \cite{zaffalon2021} in a specific setup, the result can be used to evaluate whether the EM algorithm attains the global maximum. As a future work, we intend to extend our result to the case of continuous variables.

\end{document}